\documentclass{article}

\usepackage[preprint]{neurips_2024}

\usepackage[utf8]{inputenc} 
\usepackage[T1]{fontenc}    
\usepackage{hyperref}       
\usepackage{url}            
\usepackage{booktabs}       
\usepackage{amsfonts}       
\usepackage{nicefrac}       
\usepackage{microtype}      
\usepackage{xcolor}         
\usepackage{amsmath, amssymb}
\usepackage{graphicx}
\usepackage{adjustbox}
\usepackage{threeparttable}  
\usepackage{subcaption}
\usepackage[linesnumbered,ruled,vlined]{algorithm2e}
\usepackage{amsthm} 
\newtheorem{theorem}{Theorem}

\theoremstyle{remark}

\title{Topology-Guided Knowledge Distillation for Efficient Point Cloud Processing}

\author{
  Luu Tung Hai \\
  The University of Alabama at Birmingham, USA \\
  \texttt{luutunghai@gmail.com} \\
  \And
  Thinh D. Le \\
  Soongsil University, South Korea \\
  \texttt{thomlestudy295@gmail.com} \\
  \AND 
  Zhicheng Ding \\
  Bowling Green State University, USA \\
  \texttt{dingz@bgsu.edu} \\
  \And 
  Qing Tian \\
  The University of Alabama at Birmingham, USA \\
  \texttt{qtian@uab.edu} \\
  \AND
  Truong-Son Hy \thanks{Corresponding Author} \\
  The University of Alabama at Birmingham, USA \\
  \texttt{thy@uab.edu}
}

\begin{document}

\maketitle

\begin{figure}[htbp]
    \centering

    \begin{subfigure}[t]{0.45\textwidth}
        \centering
        \includegraphics[width=\linewidth]{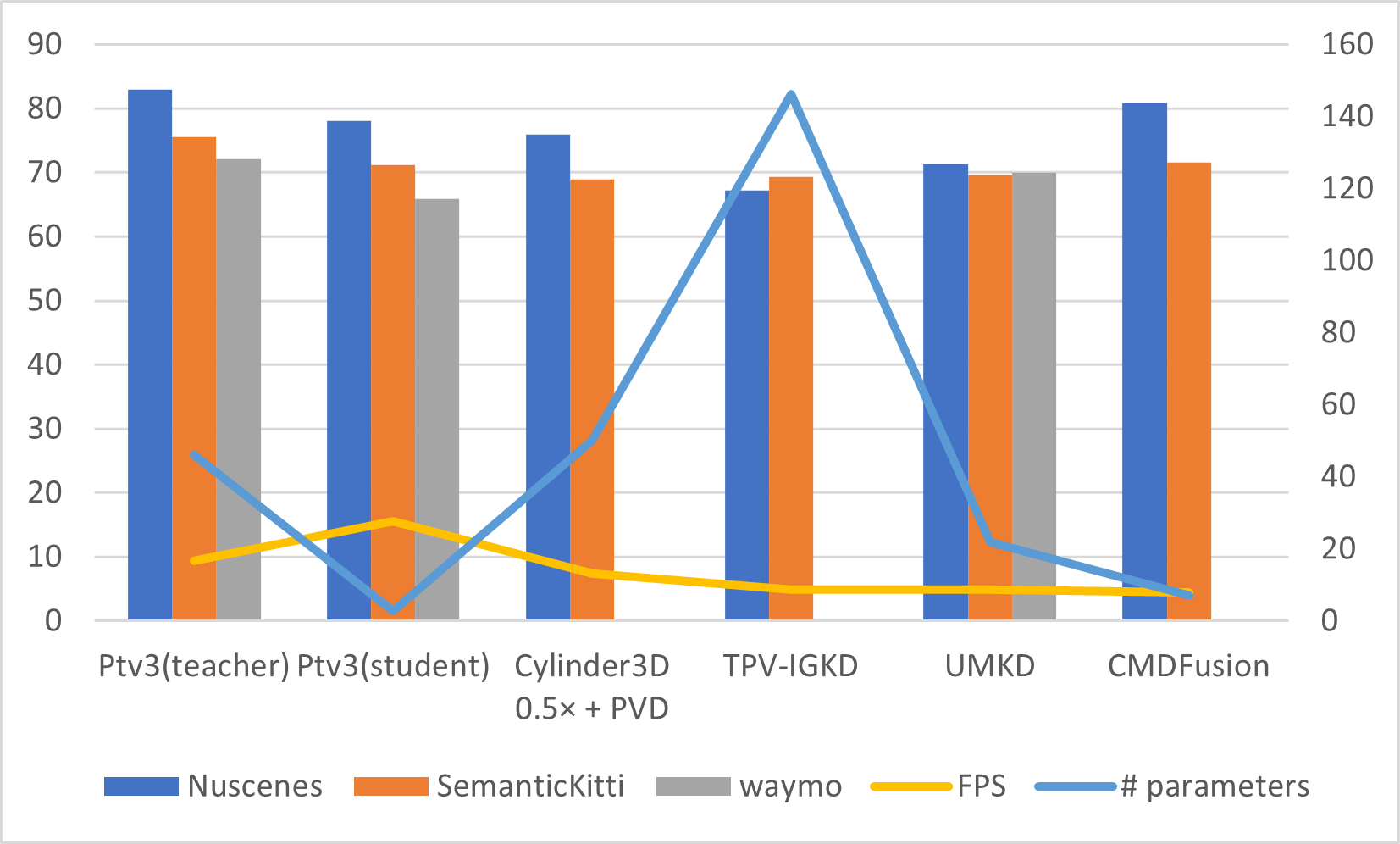}
        \caption{Performance comparison}
        \label{fig:compare-KD}
    \end{subfigure}
    \hfill
    \begin{subfigure}[t]{0.45\textwidth}
        \centering
        \includegraphics[width=\linewidth]{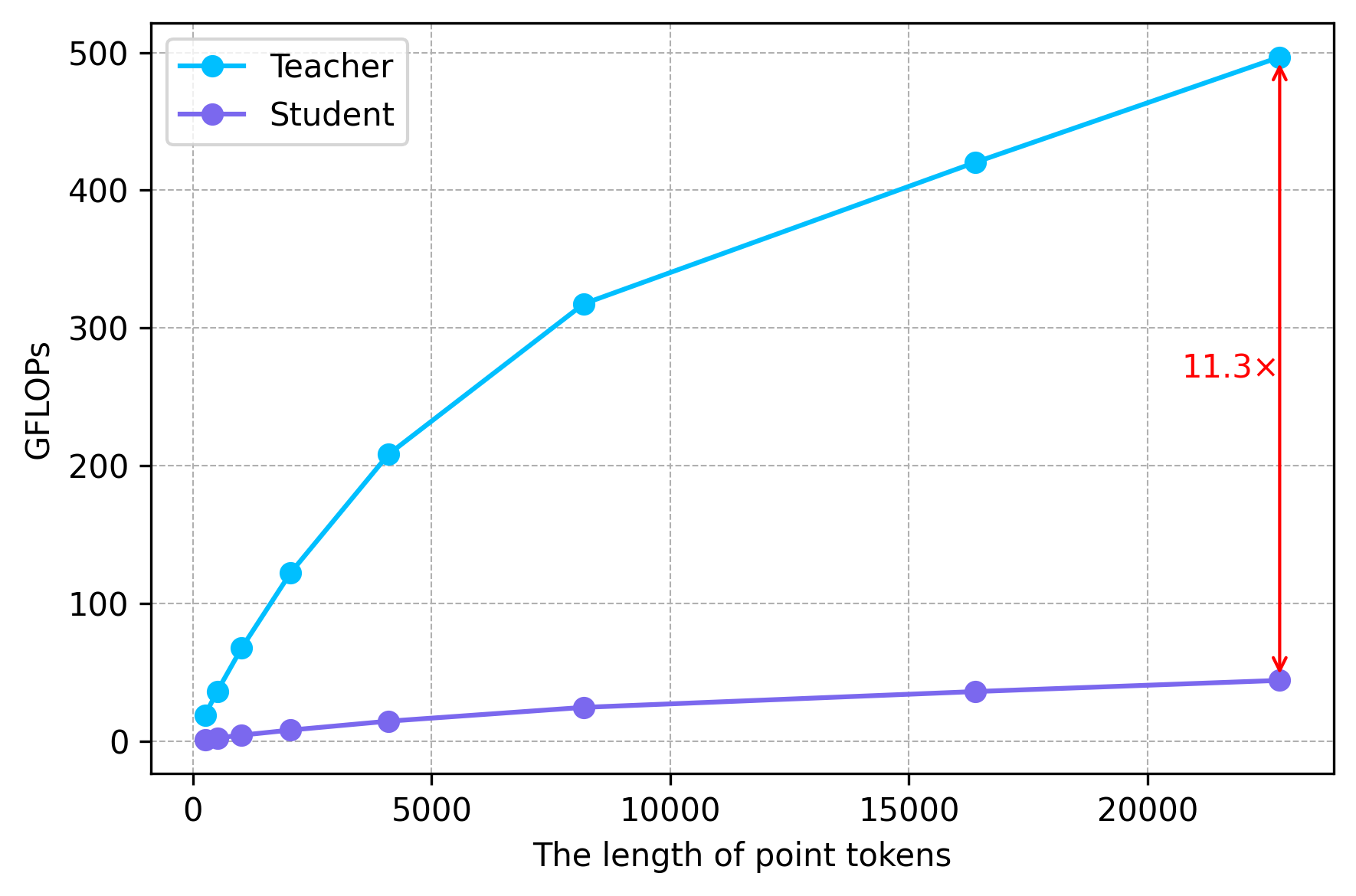}
        \caption{GFLOPs comparison}
    \end{subfigure}

    \vskip\baselineskip

    \begin{subfigure}[t]{0.45\textwidth}
        \centering
        \includegraphics[width=\linewidth]{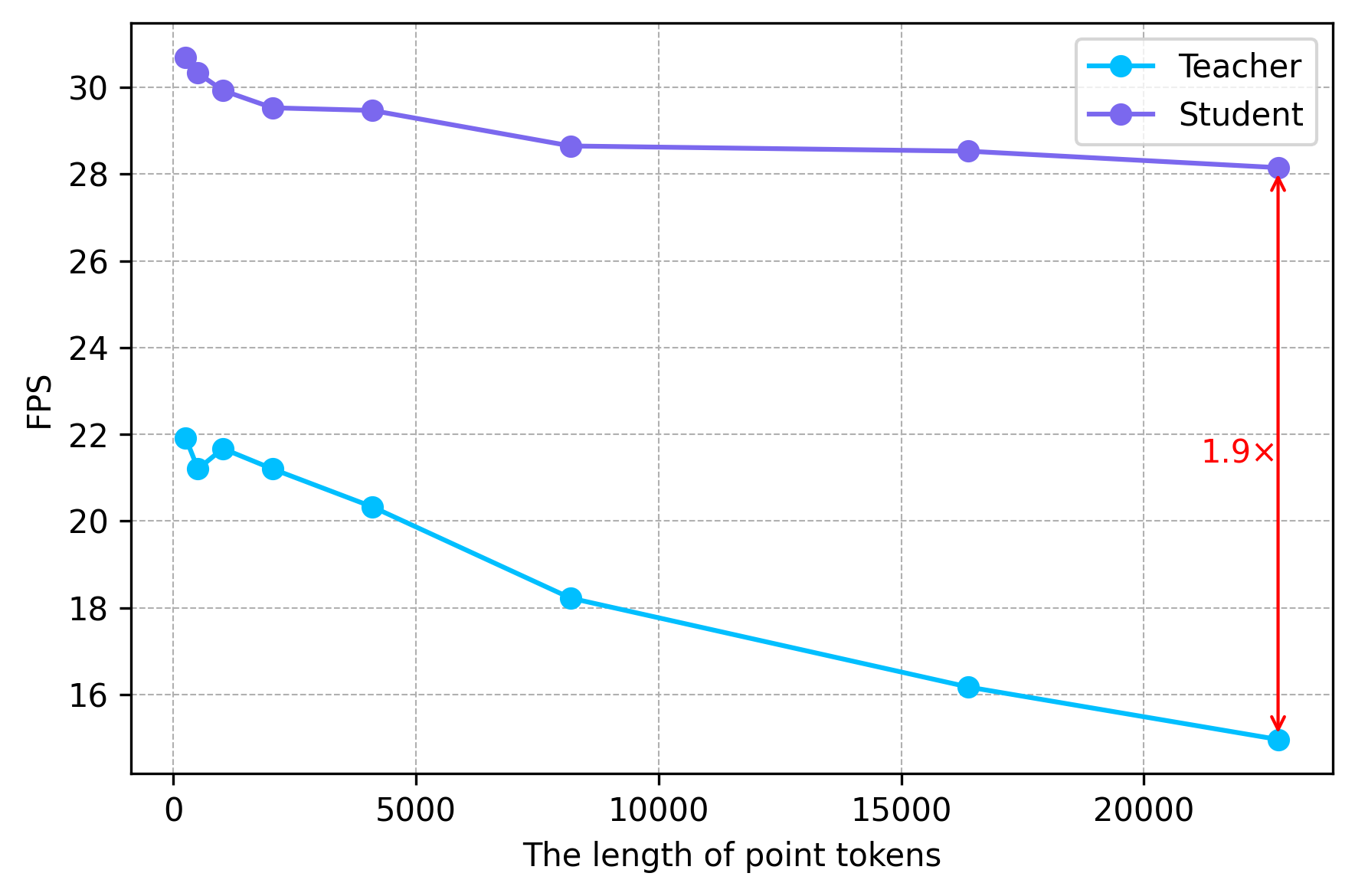}
        \caption{FPS comparison}
    \end{subfigure}
    \hfill
    \begin{subfigure}[t]{0.45\textwidth}
        \centering
        \includegraphics[width=\linewidth]{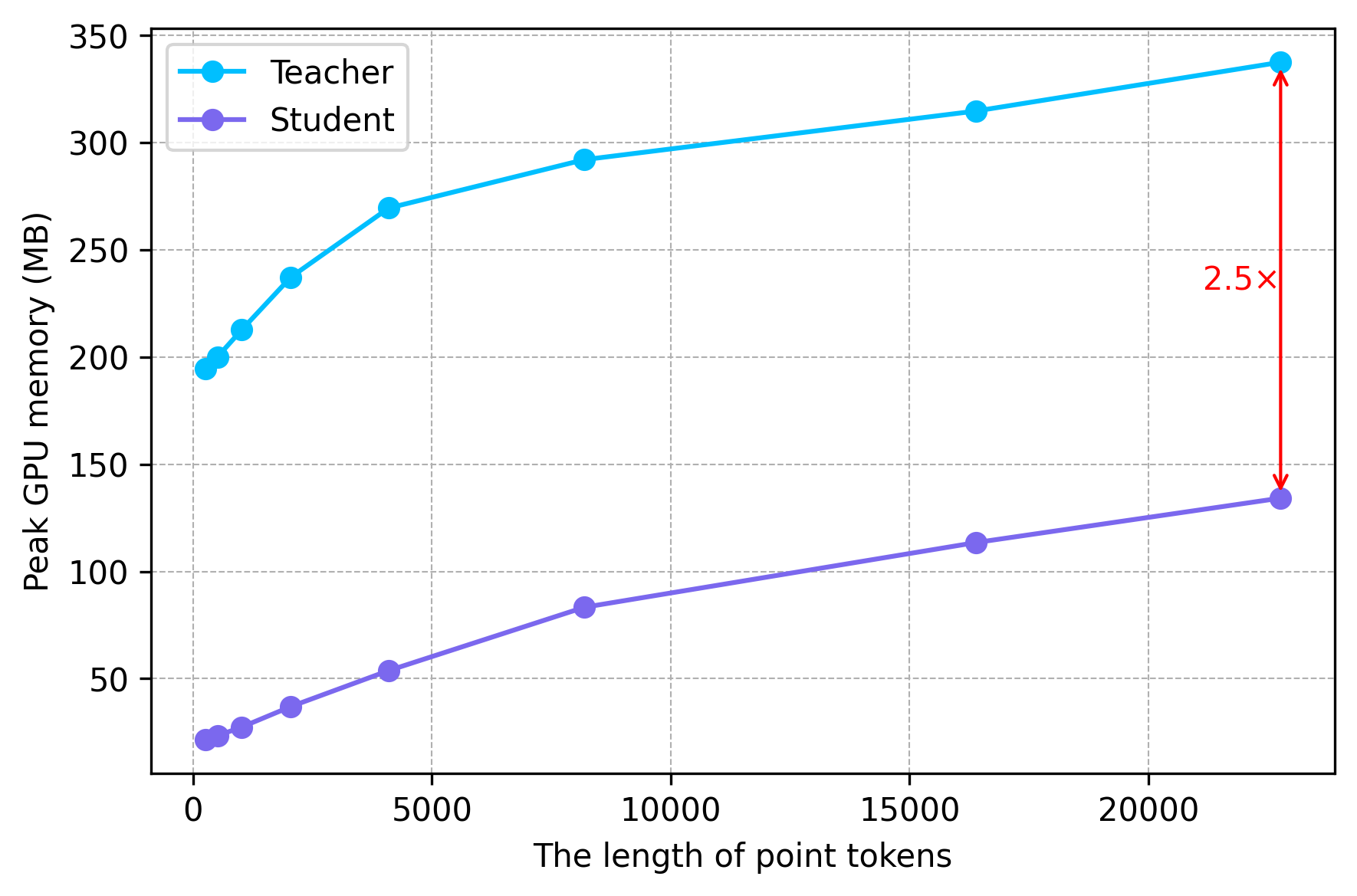}
        \caption{Peak Memory Comparison}
    \end{subfigure}

    \caption{Comprehensive comparisons between our proposed method and state-of-the-art knowledge distillation baselines across multiple evaluation metrics~\cite{point-to-voxel, tpv-igkd, umkd, cmdfusion}. 
(a) The radar chart demonstrates that our method achieves consistently better mIoU on three key datasets (NuScenes, SemanticKITTI, and Waymo), along with favorable FPS and memory efficiency. 
(b)-(d) As the input point token length increases, our approach maintains lower GPU memory usage and FLOPs, while sustaining significantly faster inference speed.
(c) Peak GPU memory usage during inference, measured using max memory allocated function in Torch. This metric reflects the highest amount of memory used by the PyTorch tensors by the caching allocator during the inference phase. Notably, this value may differ significantly from the memory reported by the PyTorch Profiler due to its inclusion of temporary allocations used by CUDA kernels.}
    \label{fig:quad_comparison}
\end{figure}

\begin{abstract}

    
    
    
    
    Point cloud processing has gained significant attention due to its critical role in applications such as autonomous driving and 3D object recognition. However, deploying high-performance models like Point Transformer V3 in resource-constrained environments remains challenging due to their high computational and memory demands. This work introduces a novel distillation framework that leverages topology-aware representations and gradient-guided knowledge distillation to effectively transfer knowledge from a high-capacity teacher to a lightweight student model. Our approach captures the underlying geometric structures of point clouds while selectively guiding the student model's learning process through gradient-based feature alignment. Experimental results in the Nuscenes, SemanticKITTI, and Waymo datasets demonstrate that the proposed method achieves competitive performance, with an approximately 16$\times$ reduction in model size and a nearly 1.9$\times$ decrease in inference time compared to its teacher model. Notably, on NuScenes, our method achieves state-of-the-art performance among knowledge distillation techniques trained solely on LiDAR data, surpassing prior knowledge distillation baselines in segmentation performance. Our implementation is available publicly at: \url{https://github.com/HySonLab/PointDistill}.
\end{abstract}

\section{Introduction}



Point cloud data are a critical representation of 3D geometry and have become essential in a wide range of applications, from autonomous driving and robotic navigation to urban mapping~\cite{voxelNet, geiger2012we, gomez2016pl, oh2002development}. Recent advances in deep learning have enabled significant progress in point cloud processing, with models such as Point Transformer V3~\cite{wu2024pointtransformerv3simpler} setting new benchmarks in accuracy and robustness. Despite the success of models like Point Transformer V3, their high computational demands and memory requirements~\cite{pointCloudCompress, 3D-Point-Cloud-Compression-A-Survey} pose challenges for deployment in resource-constrained environments, such as edge devices or real-time systems. To address this issue, various model compression strategies have been introduced, including methods such as network pruning~\cite{han2015deep, liu2019metapruning, louizos2018learning}, quantization~\cite{choi2018pact, dong2022finding, nagel2019data}, lightweight model architectures~\cite{howard2019searching, ma2018shufflenet}, and knowledge distillation~\cite{Zhang_2023_CVPR, Hou_2022_CVPR, Zhang2024HVDistill}.

Knowledge distillation is a machine learning technique that aims to transfer knowledge from a large and high capacity model to a smaller and more efficient model~\cite{hinton2014distilling, romero2014fitnets, sanh2019distilbert}. This approach allows the student model to approximate the performance of the teacher while being computationally less demanding, making it suitable for deployment in resource-constrained environments such as edge devices or mobile platforms. Over the years, knowledge distillation has been effectively applied in various domains, including image recognition~\cite{romero2014fitnets, liu2019structured} and natural language processing~\cite{hahn2019selfknowledgedistillationnaturallanguage, he2021distillersystematicstudymodel, rashid2020zeroshotknowledgedistillationnatural},  demonstrating its versatility and impact. Recently, several approaches have been introduced to incorporate knowledge distillation into 3D detection tasks using point cloud data~\cite{guo2021liga, Zhang2024HVDistill}. Nevertheless, these methods primarily emphasize the selection of student-teacher models in a multimodal context, such as utilizing an image-based teacher to guide a point-cloud-based student detector or vice versa, while largely overlooking the distinctive characteristics of point clouds. 

To address the challenges of current problems on point cloud distillation and the deployment of high-performance point cloud models in resource-constrained environments, we propose a novel distillation framework that combines topology-aware knowledge representation with gradient-guided distillation techniques. The framework leverages the inherent geometric and structural properties of point clouds to preserve critical topological information during the distillation process. By integrating gradient-based guidance, the proposed approach selectively emphasizes salient geometric features that contribute most significantly to the model's performance, enabling efficient knowledge transfer from a high-capacity teacher model to a lightweight student model. This strategy ensures that the student model retains competitive accuracy while significantly reducing computational and memory requirements, making it suitable for real-time and edge-based applications. 

Extensive experiments on the proposed method have been conducted to demonstrate the effectiveness of our approach over previous knowledge distillation methods. Our main contributions can be summarized as follows.
\begin{itemize}
    \item We propose a novel distillation framework that integrates topology-aware knowledge representation and gradient-guided distillation techniques, addressing the challenges of deploying high-performance point cloud models in resource-constrained environments.
    \item The framework leverages the unique geometric and structural properties of point clouds, embedding topological information into the distillation process to ensure the preservation of critical features necessary for accurate predictions.
    \item By incorporating gradient-guided distillation, our method selectively emphasizes salient features, enabling efficient and effective knowledge transfer from the teacher model to the student model.
    \item Extensive experimental results on popular benchmark datasets, such as Nuscenes reveal that our approach achieves up to a 16$\times$ reduction in the number of parameters and a 77.75\% reduction in CUDA memory consumption in linear operations and a 2.5$\times$ lower in peak CUDA memory usage during inference while maintaining accuracy within 5\% of state-of-the-art of non-distilled methods.
\end{itemize}
\section{Related Works}

\subsection{3D Point Cloud Processing}

The representation of 3D data using point clouds has become increasingly prominent in domains such as autonomous driving, robotics, and 3D reconstruction. Traditional deep learning approaches for understanding 3D point clouds can be categorized into three main types: projection-based, voxel-based, and point-based methods~\cite{halperin2025pointcloudbasedscene}. Projection-based techniques map 3D points onto 2D image planes and employ 2D CNN backbones for feature extraction\cite{chen2017multi, lang2019pointpillars, li2016vehicle}, often losing geometric details in the process. Voxel-based methods convert point clouds into structured voxel grids, allowing 3D convolutions with sparse convolution enhancing efficiency\cite{choy20194d, song2017semantic, wang2017cnn}, though they encounter scalability issues due to limited grid resolution, sparse and irregular data distribution, and kernel size constraints. In contrast, point-based methods directly process raw point clouds\cite{ma2022rethinking, qi2017pointnet, thomas2019kpconv, zhao2019pointweb}, with early approaches struggling to capture local structures until recent transformer-based architectures improved performance by modeling long-range dependencies and adapting to irregular distributions\cite{guo2021pct, robert2023efficient, wu2022point, yang2023swin3d}. Furthermore, hybrid methods that integrate point-voxel or graph-based representations have emerged to balance accuracy and efficiency. Across these approaches, challenges such as noise, occlusion, and varying point density in real-world data continue to impact performance.

\subsection{Point Transformer Architecture}

Transformer architectures improve point-based methods by leveraging self-attention to capture local and global dependencies effectively, outperforming CNN-based and voxel-based approaches. Early models like PCT~\cite{guo2021pct} and Point Transformer~\cite{wu2022point} demonstrated strong performance in classification and segmentation tasks.

Point Transformer V1 (PTv1)~\cite{point_trans_v1} extended the transformers to unordered 3D point sets by vector self-attention and local attention based on kNN, improving spatial modeling, but suffering from high memory and computational costs. Point Transformer V2 (PTv2)~\cite{point_trans_v2} introduced group vector attention and grid-based grouping to enhance scalability and reduce parameters, although kNN remained a bottleneck limiting long-range dependency capture.

Point Transformer V3 (PTv3)~\cite{wu2024pointtransformerv3simpler} shifted toward simplicity by serializing point clouds using space-filling curves and employing serialized patch attention, greatly expanding receptive fields, and eliminating kNN dependence. PTv3 achieved a 3.3$\times$ speedup and a 10.2$\times$ memory reduction over PTv2, establishing state-of-the-art results in diverse 3D tasks. However, PTv3's preprocessing overhead, increased latency on dense clouds, and dependence on high-end hardware limit its applicability in real-time, resource-constrained scenarios.

\subsection{Knowledge Distillation}

\textbf{Knowledge distillation (KD)} is a model-independent technique that improves student model training by transferring knowledge from a pre-trained teacher model, offering a way to enhance the efficiency of models such as Point Transformer V3 (PTv3). Early KD methods~\cite{hinton2015distilling} matched softmax outputs for classification, while later studies~\cite{liu2019structured, romero2014fitnets, tung2019similarity} extended KD to intermediate layers, capturing richer geometric and contextual information crucial for point-cloud data. KD is particularly promising for addressing the challenges of PTv3 in real-time deployment by enabling lighter, faster models.

\textbf{Topological Distillation} leverages topological data analysis (TDA) to transfer global structural features. Methods like TGD~\cite{jeon2024leveragingtopologicalguidanceimproved} and TopKD~\cite{kim2024do} distill topological knowledge through persistence images (PI) and diagrams (PD), improving the alignment of student-teacher. Despite benefits, topological distillation faces scalability challenges due to the computational cost of TDA and potential errors from PD-to-PI approximations. Its effectiveness across diverse point cloud tasks and noisy data remains limited, requiring further research for maturity.

\section{Methodology}

\subsection{Overview of the Framework}


Our framework proposes a distillation approach to develop lightweight student models for point-cloud processing, targeting both output replication and internal representation alignment. As illustrated in Figure~\ref{fig:architecture-overall}, the teacher model is a pre-trained high capacity point cloud network that extracts rich semantic and geometric features. 

The student model is trained to emulate the teacher’s behavior through two proposed mechanisms:

\begin{itemize}
    \item \textbf{Topological Distillation}: Both teacher and student feature representations undergo Topological Data Analysis (TDA) to capture global structural information. Chamfer loss is applied to align the topological signatures, encouraging the student to preserve critical geometric structures.
    
    \item \textbf{Gradient-Guided Feature Alignment}: Feature maps are compared from the teacher and student, where gradients with respect to the features guide the alignment process.
\end{itemize}

\begin{figure}[h]
    \centering
    \includegraphics[width=\textwidth]{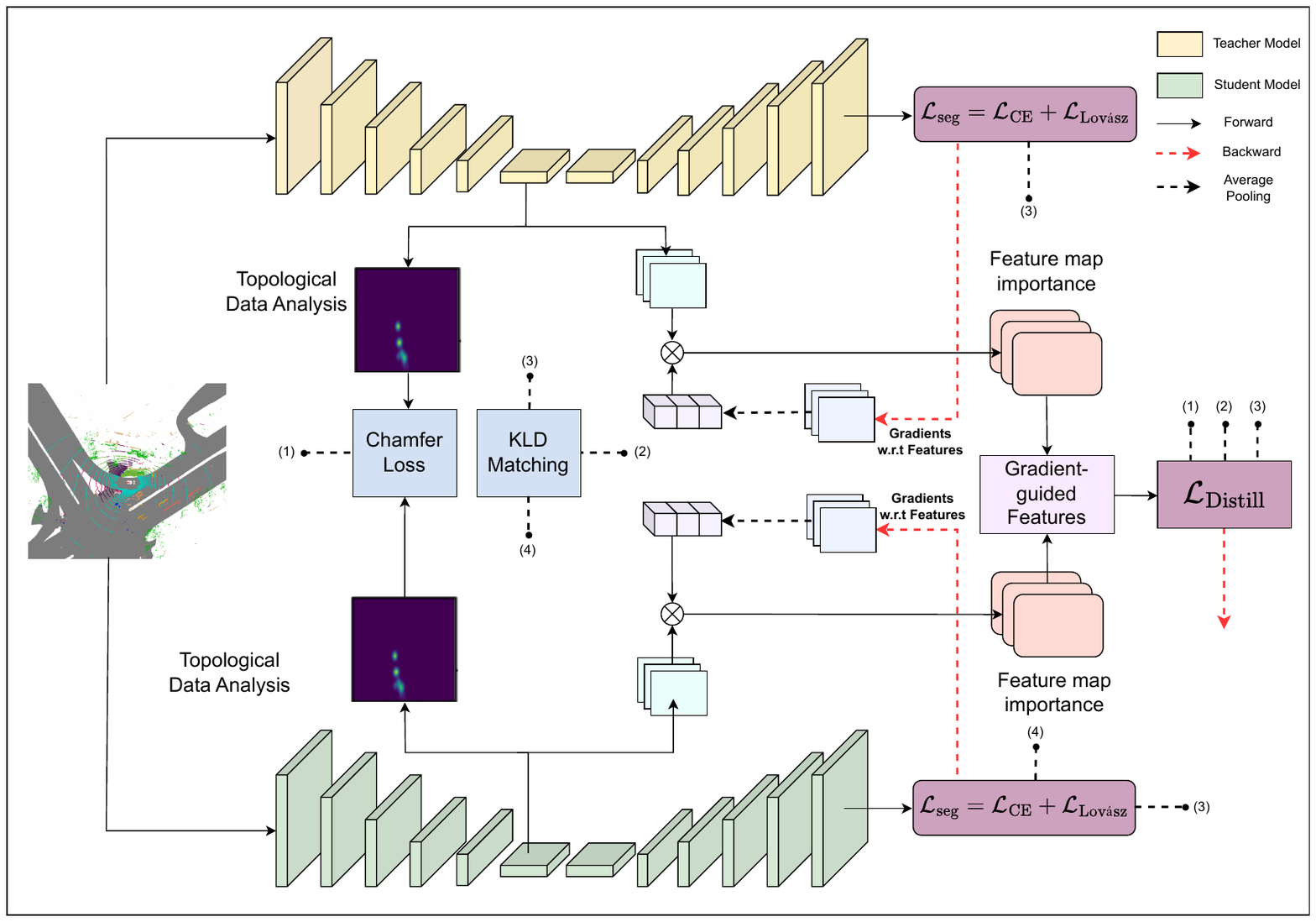}
    \caption{Overview of the proposed knowledge distillation framework for point cloud processing. The framework transfers knowledge from a teacher model (yellow) to a student model (green) using topological data analysis (TDA) (1),  KLD matching (2), student semantic segmentation loss (3), and gradient-guided feature alignment (4). The total loss is the combination of (1), (2), (3) and gradient-guided feature loss.}
    \label{fig:architecture-overall}
\end{figure}

\subsection{Topology-Aware Distillation Learning}

Traditional knowledge distillation focuses on aligning Euclidean feature maps, which may fail to capture the structural and geometric relationships inherent in 3D point clouds. To address this, we introduce a topology-aware distillation framework that ensures the student model preserves essential topological structures.

Given a point cloud $X \in \mathbb{R}^{N \times 3}$, we construct a simplicial complex $\mathcal{K}(X)$ through Vietoris-Rips filtration and extract persistence diagrams $D_T$ and $D_S$ for the teacher and the student, respectively. We then define a topology loss based on the Chamfer Distance between persistence diagrams:
\begin{equation}
\mathcal{L}_{\text{topo}} = \mathcal{L}_{\text{CD}}(D_T, D_S).
\end{equation}

This topology loss serves as a regularizer, promoting topological consistency between teacher and student models without overwhelming feature-based alignment objectives. We further bound the gradient of $\mathcal{L}_{\text{topo}}$ to ensure a stable optimization. A theoretical justification and detailed discussion comparing Chamfer Distance and Wasserstein Distance are provided in Appendix~\ref{appendix:topology-proof}.









\subsection{Gradient-Guided Knowledge Distillation}
\label{subsection:gkd}
To ensure that the student model learns the most task-relevant features from the teacher and inspired from~\cite{lan2023gradientguidedknowledgedistillationobject}, we propose a gradient-guided feature alignment mechanism for the semantic segmentation task that leverages the gradients of the task-specific loss to prioritize important features during distillation. Formally, we define the importance weight of the \(k\)-th feature channel at layer \(l\) as:
\begin{equation}
w_l^k = \frac{1}{N} \sum_{i=1}^{N} \left| \frac{\partial \mathcal{L}_{\text{task}}}{\partial F_{i,k}^l} \right|,
\end{equation}
where \(\mathcal{L}_{\text{task}}\) is the task loss (e.g., cross-entropy for semantic segmentation), and \(F_{i,k}^l \in \mathbb{R}\) is the feature activation for point \(i\), channel \(k\), and layer \(l\). These gradients are averaged over all \(N\) points to obtain a channel-wise importance score. We then scale the feature representations for each point and channel:
\begin{equation}
\tilde{F}_{i,k}^l = w_l^k F_{i,k}^l.
\end{equation}
Next, we compute a gradient-weighted feature map by aggregating the absolute values of the scaled features across all \(C\) channels for each point:
\begin{equation}
M^l_i = \sum_{k=1}^{C} |\tilde{F}_{i,k}^l|, \quad M^l = \text{Norm}([M^l_1, M^l_2, \ldots, M^l_N]),
\end{equation}
where \(\text{Norm}\) is min-max normalization, defined as \(\text{Norm}(x_i) = \frac{x_i - \min(x)}{\max(x) - \min(x)}\) for a vector \(x = [x_1, \ldots, x_N]\). The resulting feature maps for the teacher and the student are denoted \(M_T^l\) and \(M_S^l\), respectively. The gradient-guided feature alignment loss is then:
\begin{equation}
\mathcal{L}_{\text{grad}} = \frac{1}{N} \sum_{l=1}^{L} \sum_{i=1}^{N} \left| M_{i,T}^l - M_{i,S}^l \right|.
\end{equation}
This loss encourages the student to align its task-relevant features with the teacher’s, improving the effectiveness of knowledge transfer in 3D point cloud semantic segmentation.

\subsection{Overall Distillation Objective}

To complement the alignment based on topological and saliency, we incorporate Kullback-Leibler divergence (KLD)~\cite{kullback1951kullback} to improve the consistency of the distribution level between teacher and student. By applying KLD to softened output logits or intermediate features, the student is guided to mimic not only the teacher's predictions but also the underlying confidence distribution, which helps capture class relationships and enhances generalization.

By incorporating multiple feature alignment strategies, the overall distillation objective integrates topology-aware feature transfer, gradient-guided feature alignment, distribution-level matching via KLD, Chamfer distance-based distribution matching, and segmentation loss. The final distillation loss is formulated as follows:

\begin{equation}
\mathcal{L}_{\text{Distill}} = \mathcal{L}_{\text{topo}} + \lambda_1 \mathcal{L}_{\text{grad}} + \lambda_2 \mathcal{L}_{\text{KLD}} + \lambda_3 \mathcal{L}_{\text{seg}},
\end{equation}

where:
\begin{itemize}
    \item \( \mathcal{L}_{\text{topo}} \) enforces the preservation of high-level topological structures between teacher and student models.
    \item \( \mathcal{L}_{\text{grad}} \) ensures that the most important and informative features are transferred effectively.
    \item \( \mathcal{L}_{\text{KLD}} \) minimizes the discrepancy between the feature distributions using the Kullback-Leibler divergence.
    \item \( \mathcal{L}_{\text{seg}} \) represents the standard segmentation loss, ensuring that the student maintains an accurate point classification.
\end{itemize}

The hyperparameters \( \lambda_1, \lambda_2, \) and \( \lambda_3 \) control the relative importance of each loss component. By optimizing this composite objective, the student model is guided to capture both the \textbf{ global topological} and \textbf{local geometric} properties, enhancing its generalization capability in point cloud tasks.

\section{Experiments and Results}

To assess our topology-aware distillation framework, we performed experiments on three prominent autonomous driving datasets: SemanticKITTI~\cite{behley2019semantickitti}, Waymo Open Dataset~\cite{caesar2020nuscenes}, and NuScenes~\cite{sun2020waymo}. These datasets offer large-scale, real-world point-cloud sequences, ideal for benchmarking point-cloud processing techniques. We provide detailed descriptions of the datasets, training protocols, and evaluation procedures in the Appendix~\ref{appendix:experimental-details}.

\subsection{Experimental Results}

\textbf{Comparison with previous state-of-the-art LiDAR semantic segmentation models.} The results in Table~\ref{tab:compare-nusences-non-KD} highlight the performance of various previous LiDAR semantic segmentation methods compared to our proposed distillation approach on the nuScenes test dataset, with mIoU scores ranging from 65.5\% (RangeNet++)~\cite{milioto2019rangenet++} to 78.17\% (Student with KD). Among these methods, SDSeg3D~\cite{sdseg3d} achieved the highest mIoU (77.7\%), followed closely by RPVNet~\cite{rpvnet} (77.6\%) and GFNet~\cite{gfnet} (77.6\%). In particular, the knowledge-distilled (KD) version of the student model surpassed all previous approaches with 78.17\% mIoU, demonstrating the effectiveness of knowledge distillation in improving segmentation accuracy. However, despite these advances, all these models still have a lower performance than Point Transformer V3~\cite{wu2024pointtransformerv3simpler}, which achieves 83\% mIoU on the nuScenes test dataset, setting a new benchmark in LiDAR semantic segmentation.

\textbf{Comparison with state-of-the-art LiDAR knowledge distillation semantic segmentation models.} The results in Table~\ref{fig:compare-KD} compare various knowledge distillation methods on the nuScenes dataset, FPS and the number of parameters. Point Transformer V3 (teacher) achieves the highest performance with 83\% mIoU, but has a relatively high parameter count (46.16M) and a lower FPS (16.61). Among student models, CMDFusion\cite{cmdfusion} achieves the highest mIoU (80.8\%), coming closest to the teacher model while maintaining a significantly lower parameter count (7.04M) and FPS of 8. Our proposed distilled version reaches 78.01\% mIoU, outperforming Cylinder3D 0.5× + PVD~\cite{point-to-voxel} (76\%) and all other models based on KD. Additionally, it has the highest FPS (27.64), making it the most efficient model in terms of speed, while also being the most lightweight (2.78M parameters). Cylinder3D 0.5× + PVD~\cite{point-to-voxel} and TPV-IGKD~\cite{tpv-igkd} fall behind in terms of mIoU, with 76\% and 67.2\%, respectively, while also having significantly larger model sizes (50M and 146.18M). UMKD~\cite{umkd} achieves 71.3\% mIoU, slightly outperforming TPV-IGKD, but with a smaller parameter count (21.8M). In general, our proposed KD method balances accuracy, speed, and efficiency better than other knowledge distillation methods.

\begin{table*}[h]
    \centering
    \renewcommand{\arraystretch}{1.2}
    \setlength{\tabcolsep}{3pt} 
    \resizebox{\textwidth}{!}{ 
    \begin{tabular}{l|c|cccccccccccccccc}
        \toprule
        Methods & \textbf{mIoU} 
        & \rotatebox{90}{barrier} 
        & \rotatebox{90}{bicycle} 
        & \rotatebox{90}{bus} 
        & \rotatebox{90}{car} 
        & \rotatebox{90}{construction} 
        & \rotatebox{90}{motorcycle} 
        & \rotatebox{90}{pedestrian} 
        & \rotatebox{90}{traffic-cone} 
        & \rotatebox{90}{trailer} 
        & \rotatebox{90}{truck} 
        & \rotatebox{90}{driveable} 
        & \rotatebox{90}{other} 
        & \rotatebox{90}{sidewalk} 
        & \rotatebox{90}{terrain} 
        & \rotatebox{90}{manmade} 
        & \rotatebox{90}{vegetation} \\
        \midrule
        RangeNet++~\cite{milioto2019rangenet++} & 65.5  & 66.0  & 21.3  & 77.2  & 80.9  & 30.2  & 66.8  & 69.6  & 52.1  & 54.2  & 72.3  & 94.1  & 66.6  & 63.5  & 70.1  & 83.1  & 79.8 \\
        PolarNet~\cite{zhang2020polarnet} & 71.0  & 74.7  & 28.2  & 85.3  & 90.9  & 35.1  & 77.5  & 71.3  & 58.8  & 57.4  & 76.1  & 96.5  & 71.1  & 74.7  & 74.0  & 87.3  & 85.7 \\
        SalsaNext~\cite{cortinhal2020salsanext} & 72.2  & 74.8  & 34.1  & 85.9  & 88.4  & 42.2  & 72.4  & 72.2  & 63.1  & 61.3  & 76.5  & 96.0  & 70.8  & 71.2  & 71.5  & 86.7  & 84.4 \\
        Cylinder3D~\cite{zhu2021cylindrical} & 76.1  & 76.4  & 40.3  & 91.2  & 93.8  & 51.3  & 78.0  & 78.9  & 64.9  & 62.1  & 84.4  & 96.8  & 71.6  & 76.4  & 75.4  & 90.5  & 87.4 \\
        C3D\_0.5× + KA~\cite{hou2022point} & 73.9  & 74.2  & 36.3  & 88.5  & 87.6  & 47.1  & 76.9  & 78.3  & 63.5  & 57.6  & 83.4  & 94.9  & 70.3  & 73.8  & 73.2  & 88.4  & 86.3 \\
        AMVNet~\cite{AMVNet} & 76.1 & 79.8 & 32.4 & 87.4 & 90.4 & 62.5 & 81.9 & 75.3 & 72.3 & 83.5 & 65.1 & 97.4 & 67.0 & 78.8 & 74.6 & 90.8 & 87.9 \\
        2DPASS~\cite{2dpass} & 76.2 & 75.3 & 43.5 & 95.3 & 91.2 & 54.5 & 78.9 & 78.2 & 62.1 & 70.0 & 84.2 & 96.3 & 73.2 & 74.2 & 74.9 & 89.8 & 85.9 \\
        SDSeg3D~\cite{sdseg3d} & 77.7 & 77.5 & 49.4 & 93.9 & 92.5 & 54.9 & 86.7 & 80.1 & 67.8 & 65.7 & 86.0 & 96.4 & 74.0 & 74.9 & 74.5 & 86.0 & 82.8 \\
        RPVNet~\cite{rpvnet} & 77.6 & 78.2 & 43.4 & 92.7 & 93.2 & 49.0 & 85.7 & 80.6 & 66.9 & 69.4 & 80.5 & 96.9 & 73.5 & 75.9 & 76.0 & 90.6 & 88.9 \\
        GFNet~\cite{gfnet} & 76.1 & 81.1 & 31.6 & 76.0 & 90.5 & 60.2 & 80.7 & 75.3 & 71.8 & 82.5 & 65.1 & 97.8 & 67.0 & 80.4 & 76.2 & 91.8 & 88.9 \\
        SVASeg~\cite{svgseg} & 74.7 & 74.1 & 44.5 & 88.4 & 86.6 & 48.2 & 72.4 & 72.3 & 61.3 & 57.5 & 75.7 & 96.3 & 70.7 & 74.7 & 74.6 & 87.3 & 86.9 \\
        \midrule
        \textbf{Student w.o KD} & 76.08  & 76.14  & 46.66  & 89.99  & 92.18  & 40.36  & 83.90  & 78.35  & 63.18  & 68.18  & 81.74  & 96.32  & 72.78  & 73.65  & 75.39  & 89.72  & 88.77 \\
        
        \midrule
        \textbf{Student with KD} & \textbf{78.17} & \textbf{79.11} & \textbf{48.28} & \textbf{92.87} & \textbf{94.31} & \textbf{41.29} & \textbf{85.68} & \textbf{82.93} & \textbf{62.36} & \textbf{70.21} & \textbf{80.27} & \textbf{96.75} & \textbf{76.35} & \textbf{74.22} & \textbf{78.84} & \textbf{90.87} & \textbf{89.30} \\

        \bottomrule
    \end{tabular}
    } 
    \caption{Comparison of our proposed method with previous state-of-the-art LiDAR semantic segmentation methods on the nuScenes test dataset. The table reports the mean Intersection over Union (mIoU) for different models across various object categories.}
    \label{tab:compare-nusences-non-KD}
\end{table*}

\begin{table*}[h]
    \centering
    \renewcommand{\arraystretch}{1.2}
    \setlength{\tabcolsep}{3pt} 
    \resizebox{\textwidth}{!}{ 
    \begin{tabular}{l|c|ccccccccccccccccccc}
        \toprule
        Methods & mIoU 
        & \rotatebox{90}{car} 
        & \rotatebox{90}{bicycle} 
        & \rotatebox{90}{motorcycle} 
        & \rotatebox{90}{truck} 
        & \rotatebox{90}{other-vehicle} 
        & \rotatebox{90}{person} 
        & \rotatebox{90}{bicyclist} 
        & \rotatebox{90}{motorcyclist} 
        & \rotatebox{90}{road} 
        & \rotatebox{90}{parking} 
        & \rotatebox{90}{sidewalk} 
        & \rotatebox{90}{other-ground} 
        & \rotatebox{90}{building} 
        & \rotatebox{90}{fence} 
        & \rotatebox{90}{vegetation} 
        & \rotatebox{90}{trunk} 
        & \rotatebox{90}{terrain} 
        & \rotatebox{90}{pole} 
        & \rotatebox{90}{traffic} \\
        \midrule
        SalsaNext~\cite{cortinhal2020salsanext} & 59.5  
        & 91.9  & 48.3  & 38.6  & 38.9  & 31.9  & 60.2  & 59.0  & 19.4  
        & 91.7  & 63.7  & 75.8  & 29.1  & 90.2  & 64.2  & 81.8  & 63.6  & 66.5  & 54.3  & 47.4 \\
        
        KPConv~\cite{thomas2019kpconv} & 58.8  
        & 96.0  & 32.0  & 42.5  & 33.4  & 44.3  & 61.5  & 61.6  & 11.8  
        & 88.8  & 61.3  & 72.7  & 31.6  & 95.0  & 64.2  & 84.8  & 69.2  & 69.1  & 56.4  & 47.4 \\

        FusionNet~\cite{zhang2020deep} & 61.3  
        & 95.3  & 47.5  & 37.7  & 41.8  & 34.5  & 59.5  & 56.8  & 11.9  
        & 91.8  & 68.7  & 77.1  & 30.5  & 90.5  & 69.4  & 84.5  & 69.8  & 68.5  & 60.4  & 46.2 \\

        KPRNet~\cite{kochanov2020kprnet} & 63.1  
        & 95.5  & 54.1  & 47.9  & 23.6  & 42.6  & 65.9  & 65.0  & 16.5  
        & 93.2  & 73.9  & 80.6  & 30.2  & 91.7  & 64.8  & 85.7  & 69.8  & 71.2  & 58.7  & 64.1 \\

        TORNADONet~\cite{gerdzhev2021tornado} & 63.1  
        & 94.2  & 51.2  & 48.1  & 40.0  & 38.2  & 63.6  & 60.1  & 34.9  
        & 89.7  & 66.7  & 74.5  & 28.7  & 91.3  & 65.8  & 85.6  & 71.5  & 70.1  & 58.0  & 49.2 \\

        SPVNAS~\cite{tang2020searching} & 66.4  
        & 97.3  & 51.5  & 50.8  & 59.8  & 58.8  & 65.7  & 62.5  & 43.7  
        & 90.2  & 67.6  & 75.2  & 16.9  & 91.3  & 65.9  & 86.1  & 73.4  & 71.0  & 64.6  & 66.9 \\

        Cylinder3D~\cite{zhu2021cylindrical} & 68.9  
        & 97.1  & 67.6  & 50.8  & 50.8  & 58.5  & 73.7  & 69.2  & 48.0  
        & 92.2  & 65.0  & 77.0  & 32.3  & 90.7  & 66.5  & 85.6  & 72.5  & 69.8  & 62.4  & 66.2 \\

        \midrule
        \textbf{Student w.o KD} & 69.5 & 98.0 & 68.9 & 52.4 & 52.4 & 59.9 & 74.9 & 70.5 & 49.6 & 93.2 & 66.4 & 78.2 & 34.1 & 91.7 & 67.8 & 86.7 & 73.7 & 71.1 & 63.8 & 67.6 \\
        \midrule
        \textbf{Student with KD} & \textbf{74.6} & \textbf{98.3} & \textbf{74.1} & \textbf{60.3} & \textbf{60.3} & \textbf{66.6} & \textbf{79.1} & \textbf{75.4} & \textbf{58.0} & \textbf{94.3} & \textbf{72.0} & \textbf{81.8} & \textbf{45.1} & \textbf{93.1} & \textbf{73.2} & \textbf{88.9} & \textbf{78.1} & \textbf{75.9} & \textbf{69.8} & \textbf{73.0} \\

        \bottomrule
    \end{tabular}
    } 
    \caption{Comparison of semantic segmentation performance on the SemanticKITTI dataset. The table reports the mean Intersection over Union (mIoU) for different models across various object categories.}
    \label{tab:compare-semanticKitti-nonKD}
    \vspace{-0.25cm}
\end{table*}

\begin{table}[h]
    \centering
    \renewcommand{\arraystretch}{1.2}
    \setlength{\tabcolsep}{8pt} 
    \begin{tabular}{l|c|c}
        \toprule
        \textbf{Method} & \textbf{Parameters (Millions)} & \textbf{FPS} \\
        \midrule
        RangeNet++~\cite{milioto2019rangenet++} & 50.0 & 12.5 \\
        PolarNet~\cite{zhang2020polarnet} & 45.0 & 16.7 \\
        SalsaNext~\cite{cortinhal2020salsanext} & 6.7 & 23.8 \\
        Cylinder3D~\cite{zhu2021cylindrical} & 53.0 & 12.0 \\
        SalsaNext~\cite{cortinhal2020salsanext} & 6.7 & 25.0 \\
        KPConv~\cite{thomas2019kpconv} & 15.0 & 12.0 \\
        TornadoNet~\cite{gerdzhev2021tornado} & N/A & N/A \\
        SPVNAS~\cite{tang2020searching} & 1.0 & 16.0 \\
        \midrule
        PTv3 (Teacher) & 46.16 & 16.61 \\
        \textbf{Our Student} & \textbf{2.78} & \textbf{27.64} \\
    
        \bottomrule
    \end{tabular}
    \caption{Comparison of the number of parameters (in millions) and inference speed (frames per second, FPS) for different LiDAR semantic segmentation methods on Nuscenes.}
    \label{tab:params_fps}
\end{table}

We further evaluate our method on the Waymo benchmark in Table~\ref{tab:waymo-comparison}. Among methods that use both LiDAR and camera inputs (LC), UMKD (SwiftNet34)(B) achieves the highest validation mIoU of 73.0, while MSeg3D slightly leads on the test set with 70.5. In contrast, methods that use only LiDAR (L) show competitive results, with LidarMultiNet~\cite{ye2022lidarmultinet} achieving the best validation performance of 73.8, surpassing all other methods, including those that use both modalities. However, this approach uses 3D bounding boxes as an additional supervision signal
during training. Our proposed student model also performs well with a test mIoU of 71.3 and comparable validation performance with superiority in model parameter size (110.6M vs 2.78M), indicating strong generalization even when using only LiDAR data.

\begin{table}[ht]
\centering
\begin{tabular}{lcc}
\hline
\textbf{Method} & \textbf{Input} & \textbf{mIoU} (test / val) \\
\hline
MSeg3D~\cite{li2023mseg3d} & LC & 70.5 / 69.6 \\
UMKD (B)~\cite{umkd} & LC & 70.0 / 71.1 \\
UMKD (SwiftNet34)(B)~\cite{umkd} & LC & 70.6 / 73.0 \\
PMF~\cite{zhuang2021perception} & LC & - / 58.2 \\
\hline
SalsaNext~\cite{salsanext} & L & 55.8 / - \\
Realsurf~\cite{umkd} & L & 67.6 / - \\
SPVCNN++~\cite{tang2020searching} & L & 67.7 / - \\
VueNet3D~\cite{umkd} & L & 68.6 / - \\
SphereFormer~\cite{lai2023spherical} & L & - / 69.9 \\
\hline
\textbf{Ours (Student w.o KD)} & L & 68.2 / 66.5 \\ 
\textbf{Ours (Student w KD)} & L & \textbf{69.5} / \textbf{68.7} \\
\textbf{Ours (Teacher)} & L & 71.3 / 69.8\\
\hline
\end{tabular}
\caption{Quantitative Results of Different Approaches on Waymo Open Dataset. The modalities available on Waymo include LiDAR(L), and Camera(C).}
\label{tab:waymo-comparison}
\vspace{-0.25cm}
\end{table}

\subsection{Comprehensive Teacher - Student Analysis}

\begin{table}[h!]
\centering
\resizebox{\textwidth}{!}{%
\begin{tabular}{@{}p{3.8cm}p{3.8cm}p{3.8cm}p{4.2cm}@{}}
\toprule
\textbf{Metric} & \textbf{Teacher Model (46.16M params)} & \textbf{Student Model (2.78M params)} & \textbf{Comparison} \\ \midrule
Total Parameters & 46,160,000 ($\sim$46.16M) & 2,780,000 ($\sim$2.78M) & Student is 16.6$\times$ smaller \\ \midrule
Encoder Depths & (2, 2, 2, 6, 2) (14 blocks) & (1, 1, 1, 2, 1) (6 blocks) & Student has 2.33$\times$ fewer blocks \\ \midrule
Encoder Channels & (32, 64, 128, 256, 512) & (16, 16, 32, 64, 128) & Student channels 2$\times$-4$\times$ smaller \\ \midrule
Encoder Attention Heads & (2, 4, 8, 16, 32) & (1, 1, 2, 4, 8) & Student heads 2$\times$-4$\times$ fewer \\ \midrule
Decoder Depths & (2, 2, 2, 2) (8 blocks) & (1, 1, 1, 1) (4 blocks) & Student has 2$\times$ fewer blocks \\ \midrule
Decoder Channels & (64, 64, 128, 256) & (64, 64, 128, 128) & Student last stage 2$\times$ smaller \\ \midrule
Decoder Attention Heads & (4, 4, 8, 16) & (2, 2, 4, 8) & Student heads 2$\times$ fewer \\ \midrule
Patch Size & 1024 & 1024 & Same \\ \midrule

Encoder (GFLOPs) & 380.25 & 10.36 & Student is 36.70$\times$ lower \\ \midrule
Decoder (GFLOPs) & 116.44 & 33.45 & Student is 3.48$\times$ lower \\ \midrule
Total Attention Compute (Encoder) & 22.58 & 0.60 & Student is 37.63$\times$ lower \\ \midrule

Inference Time (Excl. Overhead) & $\sim$0.0592s & $\sim$0.0362s & Speedup: 1.64$\times$\\ \midrule

Batch Time Inference & $\sim$7.34s & $\sim$4.38s & Student consistently faster \\ \midrule
FPS & $\sim$16.90 & $\sim$27.70 & Student 1.64$\times$ higher FPS \\ \midrule

Fixed Overhead & $\sim$0.018s & $\sim$0.011s & Speed up: 1.58$\times$ \\ \midrule
Attention Mechanism & Flash Attention enabled & Flash Attention enabled & Student benefits more from Flash Attention \\ 
\bottomrule
\end{tabular}%
}
\caption{Comparison between Teacher and Student Models on NuScenes.}
\label{tab:comparison-nuscenes}
\end{table}


\begin{table}[h!]
\centering

\resizebox{\textwidth}{!}{%
\begin{tabular}{@{}p{3.8cm}p{3.8cm}p{3.8cm}p{4.2cm}@{}}
\toprule
\textbf{Metric} & \textbf{Teacher Model (46.16M params)} & \textbf{Student Model (2.78M params)} & \textbf{Comparison} \\ \midrule
Total CPU Time\textsuperscript{b} & 501.109 ms & 203.096 ms & 298.013 ms (2.47$\times$) \\ \midrule
Total CUDA Time & 427.305 ms & 102.068 ms & 325.237 ms (4.19$\times$) \\ \midrule
$\text{MM}_{\text{add}}$ (CUDA Memory)\textsuperscript{a} & 16.05 GB & 3.57 GB & 12.48 GB (4.5$\times$) \\ \midrule
$\text{MM}_{\text{add}}$ (Self CUDA Time)\textsuperscript{a} & 173.506 ms & 34.749 ms & 138.757 ms (5.0$\times$) \\ \midrule
$\text{Alloc}$ (CUDA Memory) & 8.73 GB & 3.00 GB & 5.73 GB (2.91$\times$) \\ \midrule
$\text{Idx}$ (CUDA Memory) & 7.27 GB & 1.78 GB & 5.49 GB (4.08$\times$) \\ \midrule
$\text{GELU}$ (CUDA Memory) & 6.82 GB & 1.67 GB & 5.15 GB (4.08$\times$) \\ \midrule
$\text{LN}$ (CUDA Memory)\textsuperscript{a} & 4.62 GB & 985.43 MB & 3.66 GB (4.69$\times$) \\ \midrule
$\text{Infer}$ (Self CPU Time) & 47.008 ms & 37.243 ms & 9.765 ms (1.26$\times$) \\ \midrule
$\text{MM}_{\text{add}}$ (Self CPU Time)\textsuperscript{a} & 5.402 ms & 2.715 ms & 2.687 ms (1.99$\times$) \\ \bottomrule
\end{tabular}%
}
\begin{tablenotes}
\small
\item \textbf{a} - Operations contributing to attention layers (e.g., matrix multiplications for $Q$, $K$, $V$ computations, and layer normalization).
\item \textbf{b} - Total CPU Time is reported from the profiling run with CUDA time measurements; a separate memory-focused run reports 278.804 ms (Teacher) and 92.436 ms (Student), yielding a 3.02$\times$ ratio.
\end{tablenotes}
\caption{Memory and Time Usage Comparison between Teacher and Student Models.}
\label{tab:memory_comparison}
\vspace{-1cm}
\end{table}

Tables~\ref{tab:comparison-nuscenes} and~\ref{tab:memory_comparison} compare the Teacher (46.16M params) and Student (2.78M params) models in NuScenes, showcasing the efficiency benefits of knowledge distillation for semantic segmentation. The Teacher model, with 16.6$\times$ more parameters than the Student, has a deeper architecture suited for high-accuracy tasks on powerful hardware. In contrast, the Student’s lightweight design, featuring fewer encoder and decoder blocks, attention heads, and channels, significantly reduces computational overhead. Specifically, the Student achieves a 36.70$\times$ reduction in encoder FLOPs and a 37.63$\times$ decrease in total attention compute, reflecting its streamlined transformer architecture. This efficiency translates to a 1.64$\times$ faster inference time (0.0362s vs. 0.0592s) and a 1.64$\times$ higher FPS (27.70 vs. 16.90), with a 1.68$\times$ reduction in batch inference time, making it suitable for real-time applications. Additionally, Student uses 4.5$\times$ less peak CUDA memory (3.57 GB vs 16.05 GB), which benefits more from Flash Attention optimizations. In terms of time and memory usage, the total CPU time of the teacher (501.109 ms) and the CUDA time (427.305 ms) are 2.47$\times$ and 4.19$\times$ higher than that of the student (203.096 ms, 102.068 ms), respectively. At the operational level, the teacher's $\text{MM}_{\text{add}}$ operation consumes 16.05 GB of CUDA memory and 173.506 ms, which are 4.5$\times$ and 5.0$\times$ more than the student's 3.57 GB and 34.749 ms. The Teacher also requires 2.91$\times$ to 4.69$\times$ more memory for operations like $\text{Alloc}$, $\text{Idx}$, and $\text{LN}$, underscoring its higher resource demands. Although the Teacher excels in accuracy on high performance hardware, the Student's reduced memory footprint and faster execution make it ideal for real-time deployment on resource-constrained edge devices, competitive with efficient non-KD models like SparseConv~\cite{choy20194d} and KPConv~\cite{thomas2019kpconv}, with potential accuracy trade-offs worth exploring further.

\section{Discussion and Future Work}
\label{section:discussion}

This work demonstrates that incorporating topological priors and gradient-guided feature alignment significantly enhances the knowledge distillation process for point-cloud semantic segmentation. Using structural insights from persistent homology and prioritizing tasks-relevant features, the proposed student model achieves a strong trade-off between accuracy and efficiency, making it highly suitable for deployment in resource-constrained settings.

However, there are several directions worth exploring further. First, while the topology-aware loss effectively captures global geometric structures, it may be sensitive to the filtration scale used in persistence diagram computation. Future work could investigate adaptive or learned filtration strategies to improve robustness across diverse scene types. Second, although the student model generalizes well across multiple datasets, its reliance on a fixed student architecture may limit flexibility. Exploring neural architecture search or adapting a task-aware model could offer additional performance gains. Finally, our method currently distills knowledge in a one-to-one teacher-student setup; extending this to multi-teacher or collaborative distillation frameworks or exploring a one-stage joint training paradigm, where the teacher and student are optimized simultaneously, could further streamline the training process and improve representation alignment could further enhance generalization, especially in complex outdoor environments.

Unlike traditional tools such as Ripser++~\cite{zhang2020gpu}, which construct Vietoris-Rips complexes through a full filtration from $\epsilon = 0$ to a maximum threshold, our implementation approximates topological characteristics at multiple fixed scales to ensure efficiency and differentiability within neural pipelines. Although this snapshot-based approach may not capture all intermediate birth-and-death pairs, it provides sufficient coverage across representative scales while keeping the computational overhead manageable.


\bibliographystyle{plain}
\bibliography{main}

\appendix
\section{Experimental Details}
\label{appendix:experimental-details}
\subsection{Dataset Details}

\textbf{SemanticKITTI}~\cite{behley2019semantickitti} provides LiDAR point clouds from urban and suburban scenes, featuring 22 sequences with dense semantic annotations across 19 classes (e.g., vehicles, pedestrians, roads). Its high resolution and detailed labels make it a rigorous testbed for semantic segmentation.

\textbf{NuScenes}~\cite{caesar2020nuscenes} integrates LiDAR, camera, and radar data in 1,000 diverse scenes, including urban roads and highways. With 3D bounding box annotations for 23 object types, it challenges models with varied weather, occlusions, and dynamic elements suited for detection and segmentation tasks. In addition, we use nuScenes-lidar seg, which is an extension of nuScenes. This dataset has semantic labels of 32 categories and annotates each point from keyframes in nuScenes. We used the 700 scenes in the training set with segmentation labels to fine-tune for the semantic segmentation task, and the 150 scenes in the validation set to verify the performance.

\textbf{Waymo Open Dataset}~\cite{sun2020waymo} delivers high-resolution LiDAR data from 1,000 segments in various locations, with frequent sweeps and 3D annotations for vehicles, pedestrians and cyclists. Its long-range scans and varied conditions test robustness and generalization.

\textbf{Training Details.} We apply the same setting to all datasets mentioned with a batch size of 12 for training, 18 for validation, and 1 for testing. The training process is trained with 50 epochs, with evaluations performed at every epoch. We use the AdamW optimizer with an initial learning rate of 0.002 and a weight decay of 0.005. The training process follows a OneCycleLR learning rate scheduling strategy, which dynamically adjusts the learning rate throughout the training cycle. Initially, the learning rate increases rapidly to a pre-defined maximum value during the warm-up phase, ensuring a stable convergence. Then it follows a cosine annealing schedule, gradually decreasing to a much lower value as the training progresses. This approach helps the model escape sharp local minima early in training, while allowing fine-tuning in later stages for better generalization. In addition, a cyclical adjustment to weight decay prevents overfitting and improves the robustness of the model. Data augmentation techniques are applied during training, such as random rotation, scaling, flipping, and jittering. The evaluation pipeline includes a semantic segmentation evaluator and a precise evaluator, ensuring a reliable model evaluation.

\subsection{Model and Training Hyperparameters}
\label{appendix:setting}
In our experiments, we use Point Transformer V3~\cite{wu2024pointtransformerv3simpler} as both the teacher and the student backbone, the student model being significantly reduced in capacity to enhance efficiency while maintaining essential representational power. Note that at the time of this project, the authors have not yet released the final weights for the model Point Transformer V3, so we have trained them from scratch. The student network is approximately 20\% the depth of the teacher. Specifically, the encoder of the student model has shallower depths, which reduce from teacher $(2, 2, 2, 6, 2)$ to $(1, 1, 1, 2, 1)$. Similarly, the channel dimensions are reduced from $(32, 64, 128, 256, 512)$ in the teacher to $(16, 16, 32, 64, 128)$ in the student. The number of attention heads in the transformer layers is systematically reduced from the teacher’s (2, 4, 8, 16, 32) to (1, 1, 2, 4, 8) across the five encoder stages, effectively halving the complexity of multi-head attention at each stage. This reduction reduces the computational load while preserving the transformer’s capacity to model spatial relationships. The decoder follows a similar strategy, with depths of (1, 1, 1, 1), channels of (64, 64, 128, 128) and attention heads scaled down from the teacher’s (4, 4, 8, 16) to (2, 2, 4, 8) in its four stages, ensuring a proportional decrease. By retaining strides of (2, 2, 2, 2) and a patch size of 1024, the student model maintains structural compatibility for effective distillation.

Training is carried out on multiple datasets, including nuScenes, SemanticKITTI, and Waymo, using cross-entropy and Lovász segmentation losses to optimize segmentation performance. We adopt a two-stage distillation strategy in which the teacher model is first trained to full performance before being used to guide the student model in a separate distillation phase. This approach ensures that the student benefits from a fully converged and stable teacher during knowledge transfer.

For topology-aware learning, we use Vietoris-Rips filtration to compute persistence diagrams, enabling robust topological feature extraction across point-cloud datasets. This choice ensures that meaningful topological structures are captured, while preventing excessive noise in the persistence diagrams. The filtration scale was empirically determined to balance computational efficiency and representational fidelity. Using this set-up, we ensure that the student model effectively learns both the geometric and topological structures necessary for accurate point-cloud segmentation.

\subsection{Data Augmentation} 
To enhance the robustness and generalization ability of the model, we applied a series of data augmentation techniques during training. Specifically, the input point clouds were randomly rotated around the z-axis within a range of $\pm1^\circ$ with a probability of 0.5, and uniformly scaled by a random factor between 0.9 and 1.1. Random flipping was performed along spatial axes with a probability of 0.5 to introduce geometric variability. Additionally, Gaussian jittering with a standard deviation of 0.005 and a clipping value of 0.02 was applied to perturb point positions slightly. Following these augmentations, a grid sampling operation with a grid size of 0.05\,m was used to downsample the point cloud, where hashing was performed using the Fowler–Noll–Vo (FNV) hash function. Finally, the processed data were converted into tensors and relevant features (coordinates and strength) along with labels (segment) were collected for training. These augmentations were designed to simulate realistic sensor noise and spatial variations, thereby improving the model’s performance on unseen data.

\section{Details of Resources Used}
\label{appendix:compute-details}
We conducted all experiments on the University HPC Cluster using NVIDIA A100. Each node is equipped with 2 NVIDIA A100 GPUs (81 GB VRAM each), and we utilized one node and one A100 GPU for training and evaluation. The cluster runs on a Linux environment with Slurm for job scheduling.

\section{Theoretical Justification of Topology-Aware Distillation}
\label{appendix:topology-proof}

\subsection{Convergence of Chamfer Distance to 2-Wasserstein Distance}

In our framework, we adopt the Chamfer distance (\(\mathcal{L}_{\text{CD}}\)) to measure the similarity between persistence diagrams due to its efficiency and differentiability. Here, we theoretically justify this choice.

\begin{theorem}
Let \(D_T\) and \(D_S\) be the persistence diagrams for the teacher and student models, respectively. If \(\mathcal{L}_{\text{CD}}(D_T, D_S) \to 0\), then the 2-Wasserstein distance between \(D_T\) and \(D_S\) also converges:
\begin{equation}
W_2(D_T, D_S) \leq \sqrt{\mathcal{L}_{\text{CD}}(D_T, D_S)}.
\end{equation}
Thus, minimizing Chamfer Distance implicitly minimizes the Wasserstein distance between persistence diagrams, ensuring topological consistency.
\end{theorem}

\begin{proof}
By the properties of optimal transport~\cite{villani2009optimal}, the 2-Wasserstein distance between persistence diagrams satisfies:
\begin{equation}
W_2^2(D_T, D_S) = \inf_{\gamma \in \Gamma(D_T, D_S)} \sum_{(p, q) \in \gamma} \| p - q \|^2,
\end{equation}
where \(\Gamma(D_T, D_S)\) denotes the set of all valid matchings. The Chamfer Distance relaxes this formulation by independently matching each point to its nearest neighbor:
\begin{equation}
\mathcal{L}_{\text{CD}}(D_T, D_S) = \sum_{p \in D_T} \min_{q \in D_S} \| p - q \|^2 + \sum_{q \in D_S} \min_{p \in D_T} \| q - p \|^2.
\end{equation}
Since \(\mathcal{L}_{\text{CD}}\) considers all bidirectional nearest neighbors, it provides an upper bound on \(W_2^2(D_T, D_S)\). Taking the square root completes the proof.
\end{proof}

\subsection{Practical Motivation for Using Chamfer Distance}

Although the 2-Wasserstein distance (\(W_2\)) is the standard metric to compare persistence diagrams in topological data analysis, we opt for the Chamfer Distance (\(\mathcal{L}_{\text{CD}}\)) due to the following reasons:

\begin{itemize}
    \item \textbf{Computational Efficiency}: Computing \(W_2\) requires solving an optimal transport problem with complexity \(O(n^3 \log n)\), which is prohibitively expensive for large persistence diagrams derived from dense 3D point clouds. In contrast, the Chamfer distance can be computed in \(O(nm)\) time via nearest-neighbor search.
    \item \textbf{Differentiability}: The Chamfer distance is readily differentiable, enabling direct integration with gradient-based optimization. Wasserstein distance typically requires approximations (e.g., Sinkhorn regularization), introducing additional hyperparameters and potential training instability.
    \item \textbf{Empirical Stability}: In our experiments, Chamfer Distance yields stable convergence during training and maintains consistent topological structures without the need for complex approximations.
\end{itemize}

\subsection{Controlling Topology Loss Influence}

To prevent the topology loss \(\mathcal{L}_{\text{topo}}\) from dominating the overall training dynamics, we impose a gradient norm constraint:
\begin{equation}
\| \nabla_x \mathcal{L}_{\text{topo}} \| \leq \alpha \| \nabla_x \mathcal{L}_{\text{feat}} \|,
\end{equation}
where \(\alpha > 0\) is a small hyperparameter. This ensures that topology-aware regularization complements rather than overwhelms feature-based alignment.

\subsection{Visualization of Topology-Aware Analysis}

Figure~\ref{fig:topology_analysis} illustrates how persistent homology captures the evolution of topological features across different filtration scales.
Given a point cloud, we construct a simplicial complex and track the birth and death of topological structures as the filtration parameter \(\epsilon\) increases.
The persistence diagram \(D = \{(b_i, d_i)\}_{i=1}^M\) quantifies these events, where each point represents a topological feature.
Longer bars correspond to persistent structures that encode essential geometric patterns, while shorter bars typically represent noise or minor perturbations.

Visualizing the persistence diagrams allows us to better understand the types of geometric features captured by the teacher model, such as connected components (\(H_0\)), loops (\(H_1\)) and voids (\(H_2\)).
By encouraging the student to mimic these persistent topological features through topology-aware distillation, we aim to transfer not only semantic knowledge but also the critical underlying geometric structures necessary for robust point-cloud understanding.
This visualization supports the intuition behind our method, showing that topological summaries can effectively reflect meaningful geometric information beyond what is captured by the Euclidean feature alignment.

\begin{figure}[h]
    \centering
    \includegraphics[width=\linewidth]{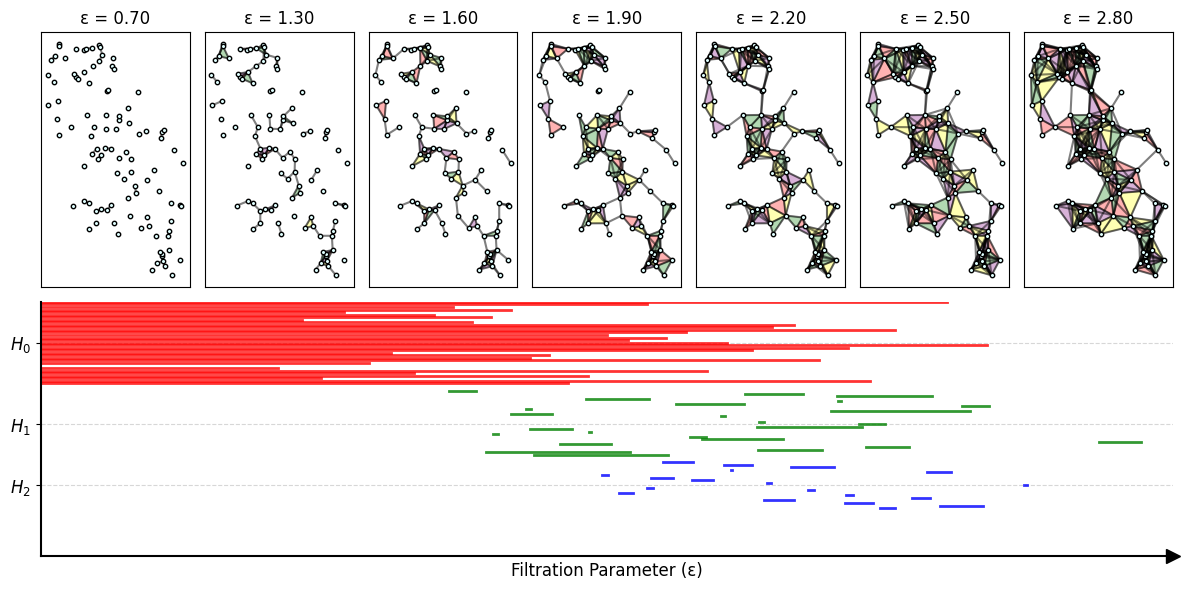}
    \caption{Illustration of topology-aware analysis through Vietoris–Rips filtration. 
    The top row depicts the evolution of the simplicial complex as the filtration parameter \(\epsilon\) increases. 
    The bottom part shows the corresponding barcode representation of persistent homology groups in different dimensions (\(H_0, H_1, H_2\)).}
    \label{fig:topology_analysis}
\end{figure}

\section{Broader Impacts}
\label{appendix:broader-impacts}
Our work enables efficient point cloud processing on resource-constrained edge devices, achieving a 1.64$\times$ faster inference (Table~\ref{tab:params_fps}) and 4.5$\times$ lower memory usage (Table~\ref{tab:memory_comparison}). This facilitates real-time deployment in self-driving cars, potentially enhancing road safety through better obstacle detection and reducing costs to make autonomous vehicles more accessible. However, the student model’s reduced accuracy (78.17\% mIoU vs. 80.03\% compared to the teacher's mIoU) may lead to errors in object detection, risking accidents if not carefully validated. Widespread adoption of autonomous vehicles might displace jobs in transportation, impacting drivers.

\section{Ablation Study}




Fig~\ref{fig:visualization} shows a top-down view of the semantic segmentation results in the NuScene validation dataset. The student model demonstrates strong alignment with the ground truth across nearly all object classes, effectively capturing the spatial layout and fine-grained structures in the scene. Compared to the teacher, the student produces cleaner boundaries and more consistent predictions, particularly in regions with small or scattered objects. This highlights the effectiveness of our knowledge distillation approach in transferring knowledge while enhancing prediction quality.

\begin{figure}[ht]
    \centering
    \begin{subfigure}[b]{0.3\textwidth}
        \includegraphics[width=\textwidth]{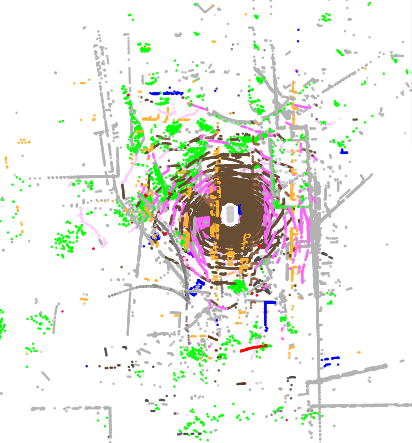}
        \caption{Ground Truth}
        \label{fig:sub1}
    \end{subfigure}
    \hfill
    \begin{subfigure}[b]{0.3\textwidth}
        \includegraphics[width=\textwidth]{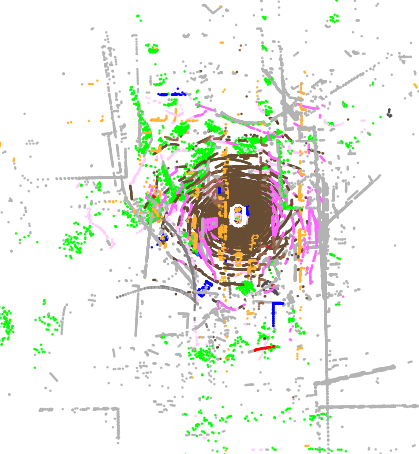}
        \caption{Teacher Prediction}
        \label{fig:sub2}
    \end{subfigure}
    \hfill
    \begin{subfigure}[b]{0.3\textwidth}
        \includegraphics[width=\textwidth]{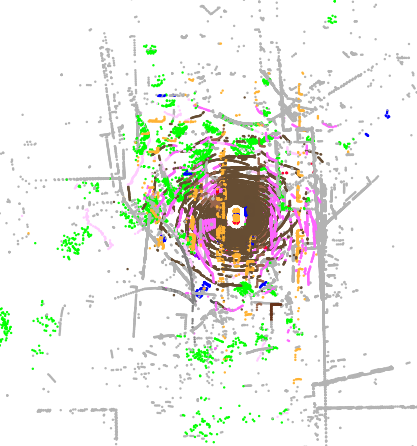}
        \caption{Student Prediction}
        \label{fig:sub3}
    \end{subfigure}
    \caption{Visualization of our method on the nuScenes validation set. (a) Ground truth, (b) teacher model prediction, and (c) student model prediction. The student model closely follows the teacher’s output and ground truth, successfully capturing almost all object classes, demonstrating the effectiveness of the knowledge distillation process.}
    \label{fig:visualization}
\end{figure}

The Table~\ref{tab:loss-component-influence} clearly demonstrates the effectiveness of each proposed loss component in improving student model performance. Adding the loss with topology awareness $\mathcal{L}_{\text{topo}}$ to the baseline leads to significant improvements, particularly on Waymo (+3.5 mIoU) and SemanticKITTI (+1.7 mIoU), where the capture of the global geometric context is essential. Meanwhile, gradient-guided feature alignment loss $\mathcal{L}_{\text{grad}}$ yields consistent gains across datasets, with the most noticeable impact on nuScenes (+0.9 mIoU), highlighting its strength to refine local features and object boundaries. When combined in the full loss formulation, the student achieves the highest mIoU on all benchmarks, confirming the complementary nature of the preservation of global topology and the alignment of local characteristics. These results validate the ability of the proposed framework to distill both structural and task-relevant knowledge, allowing the lightweight student model to approach or even surpass the state-of-the-art performance while maintaining high efficiency.

\begin{table}[h]
\centering
\begin{tabular}{ccccccc}
\toprule
$\mathcal{L}_{\text{KLD}}$ & $\mathcal{L}_{\text{seg}}$ & $\mathcal{L}_{\text{topo}}$ & $\mathcal{L}_{\text{grad}}$ & SemanticKITTI & Waymo & nuScenes \\
\midrule
$\checkmark$ & $\checkmark$ &  &  & 71.5 & 64.8 & 74.0 \\
$\checkmark$ & $\checkmark$ & $\checkmark$ &  & 73.2 & 68.3 & 77.3 \\
$\checkmark$ & $\checkmark$ &  & $\checkmark$ & 72.3 & 66.5 & 74.9 \\
$\checkmark$ & $\checkmark$ & $\checkmark$ & $\checkmark$ & 74.6 & 69.5 & 78.1 \\
\bottomrule
\end{tabular}
\caption{Influence of each component on the final performance.}
\label{tab:loss-component-influence}
\end{table}

\end{document}